\documentclass{article}

\PassOptionsToPackage{numbers, square}{natbib}

\usepackage{arxiv}

\usepackage[utf8]{inputenc} 
\usepackage[T1]{fontenc}    
\usepackage{hyperref}       
\usepackage{url}            
\usepackage{booktabs}       
\usepackage{amsfonts}       
\usepackage{nicefrac}       
\usepackage{microtype}      
\usepackage[x11names]{xcolor}         

\usepackage{graphicx}
\usepackage{amsmath}
\usepackage{bbm}
\usepackage{amsthm}
\theoremstyle{plain}
\newtheorem{theorem}{Theorem}[section]
\newtheorem{proposition}[theorem]{Proposition}
\usepackage{multirow}
\usepackage{makecell}
\usepackage{threeparttable}
\usepackage{pifont}
\newcommand{\greencheck}{\textcolor{Green4}{\ding{52}}}
\newcommand{\redcross}{\textcolor{Firebrick4}{\ding{56}}} 
\usepackage{algorithm}
\usepackage{algorithmicx}
\usepackage{algpseudocode}

\title{TMN: A Lightweight Neuron Model \\ for Efficient Nonlinear Spike Representation}

\author{Yiwen Gu, Junchuan Gu, Haibin Shen, Kejie Huang \thanks{Corresponding author.} \\
College of Information Science \& Electronic Engineering\\
Zhejiang University\\
38 Zheda Road, Hangzhou, Zhejiang Province 310027, China\\
}

\begin{document}

\maketitle

\begin{abstract}
Spike trains serve as the primary medium for information transmission in Spiking Neural Networks, playing a crucial role in determining system efficiency. Existing encoding schemes based on spike counts or timing often face severe limitations under low-timestep constraints, while more expressive alternatives typically involve complex neuronal dynamics or system designs, which hinder scalability and practical deployment. To address these challenges, we propose the Ternary Momentum Neuron (TMN), a novel neuron model featuring two key innovations: (1) a lightweight momentum mechanism that realizes exponential input weighting by doubling the membrane potential before integration, and (2) a ternary predictive spiking scheme which employs symmetric sub-thresholds $\pm\frac{1}{2}v_{th}$ to enable early spiking and correct over-firing. Extensive experiments across diverse tasks and network architectures demonstrate that the proposed approach achieves high-precision encoding with significantly fewer timesteps, providing a scalable and hardware-aware solution for next-generation SNN computing.
\end{abstract}


\section{Introduction}

Spiking Neural Networks (SNNs), recognized as the third generation of neural network models, are inspired by the biological structure and functionality of the brain \cite{Maass1997snn}. Unlike traditional Artificial Neural Networks (ANNs), which rely on real-valued activations, SNNs utilize discrete spiking events. This enables SNNs to capture temporal dynamics and process information in a manner that closely resembles brain activity \cite{Aboozar2020snnreview}. The event-driven nature of SNNs aligns with the brain's energy-efficient computational paradigm, offering potential for more efficient and low-power computing systems \cite{Yamazaki2022snnreview}.

Spike trains constitute the fundamental information carriers in SNNs, enabling a temporal distribution of information. Developing efficient spike encoding patterns is thus essential for optimizing SNN performance. Various coding schemes, such as rate coding and temporal coding, have been proposed to describe neural activity \cite{Johansson2004ttfs,Thorpe1998rank,Tim2008relative}. Rate coding represents information by the spike count (frequency), while temporal coding encodes information through the precise timing and patterns of spikes \cite{Yang2023lcttfs,Han2020tsc}. For instance, the commonly used Time-to-First-Spike (TTFS) coding scheme represents information inversely with respect to spike latency \cite{Stanojevic2022exact,Rueckauer2018ttfs}.

However, using either spike counts or timing for information representation leads to a linear scaling of distinguishable encoded values with the number of timesteps. This inherently limits the expressive power of SNNs under low-timestep constraints. A natural solution is to apply temporal importance to inputs, where combinations of differently weighted spikes enable richer value encoding. Early explorations \cite{Rueckauer2021pattern,Jaehyun2018phase,Christoph2021fs} have demonstrated the feasibility of this direction with significantly reduced timesteps, but also revealed several key challenges: 

(1) \textbf{Increased system complexity:} Time-varying weights must be controlled and applied, and the firing thresholds of neurons must dynamically adapt to these weights. 

To address this, we observe that information extraction from spike trains fundamentally depends on neuron's decoding mechanism. We innovatively introduce a \textit{momentum term into neuronal dynamics to achieve weighted decoding.} This term doubles the membrane potential before input integration, enabling a stepwise weighting process that is mathematically equivalent to applying exponentially decaying weights. By modeling the relative rather than absolute importance of historical inputs, our approach eliminates the need for time-varying thresholds. Importantly, the momentum mechanism can be realized with simple wiring between hardware units, incurring negligible resource and energy overhead.

(2) \textbf{Inefficient neural computation:} The uncertainty of future input pattern and the uneven encoding capacity of the spikes pose challenges for timely and accurate encoding. A commonly adopted mitigation strategy, drawn from TTFS coding \cite{Yang2023lcttfs,Rueckauer2018ttfs,Stanojevic2022exact,park2020t2fsnn}, is to use a pre-charge phase to reduce the input uncertainty. While pipeline scheduling can hide some latency, long pre-charge times degrade throughput and remain a critical bottleneck. 

\begin{table}[tb]
    \caption{Common symbols in this paper.}
    \centering
    \resizebox{\textwidth}{!}{
        \begin{threeparttable}
            \begin{tabular}{ll|ll|ll}
                \toprule
                \textbf{Symbol} & \textbf{Definition} & \textbf{Symbol} & \textbf{Definition} & \textbf{Symbol} & \textbf{Definition} \\
                \midrule
                $i,j$ & neuron index & $s[t]$ & spike sequence & $u[t]$ & membrane potential after reset \\
                $l$ & layer index & ${v}_{th}$ & (full) threshold & $\hat{u}[t]$ & membrane potential before reset \\
                $T$ & encoding timesteps & $\tilde{v}_{th}$ & predictive threshold & $z[t]$ & integrated inputs (PSP)\tnote{1} \\
                \bottomrule
            \end{tabular}
            \begin{tablenotes}
                \footnotesize
                \item[1] Postsynaptic potential
            \end{tablenotes}
        \end{threeparttable}
        \label{tab: symbol}
    }
\end{table}

To overcome this limitation, we introduce a Ternary Predictive Spiking (TPS) mechanism. A pair of predictive thresholds $\pm\tilde{v}_{th}$---each being half the magnitude of the full threshold $v_{th}$---are used to trigger spikes, while the reset amount remains $v_{th}$. This configuration enables \textit{predictive spiking when half the required information is received,} and the dual thresholds form \textit{a feedback system that corrects potential over-firing}. We show this mechanism reduces the required pre-charge length and enhances system efficiency.

Building upon these characteristic dynamics, we refer to the neuron model as the Ternary Momentum Neuron (TMN). The generated spike pattern share structural and weighting similarities with the canonical signed digit representation, and is thus referred to as the Canonical Signed Spike (CSS) encoding. 

The main contribution of this work lies in a novel and efficient nonlinear spike representation which compresses the number of timesteps to logarithmic scale. Temporal weighting is realized through the lightweight yet effective TMN dynamics, reducing system complexity and enhancing overall efficiency. We validate the proposed method through extensive experiments across various tasks and network architectures, demonstrating its effectiveness and generalizability.


\section{Preliminary}


\subsection{Spiking Neuron Models}

Spiking neurons serve as the fundamental computational units of SNNs. Each incoming spike contributes to the postsynaptic neuron's membrane potential. When the potential exceeds a predefined threshold, the neuron emits a spike. A general discretized spiking neuron model can be expressed as follows:
\begin{equation}
    \left\{\begin{aligned}
        & \hat{u}_i^l[t]=\mathcal{D}(u_i^l[t-1],z_i^l[t]) \\
        & s_i^l[t]=\mathcal{H}(\hat{u}_i^l[t]-v^l_{th}) \\
        & u_i^l[t] =\hat{u}_i^l[t]-v^l_{th}s[t]
    \end{aligned}
    \right.
    \label{eq: spiking neuron model}
\end{equation}
Here, $i$ denotes the neuron index and $l$ the layer index. $v^{l}_{th}$ denotes the firing threshold and $\hat{u}_{i}^{l}[t],u_i^l[t]$ represent the membrane potential before and after reset, respectively. $\mathcal{D}(\cdot)$ defines the generalized membrane dynamics, which integrates past potential and new input. The integrated input at time $t$ is given by:
\begin{equation}
    z_{i}^{l}[t]=v^{l-1}_{th}\sum_{j}w_{ij}^{l}s_{j}^{l-1}[t]+b_{i}^{l}
    \label{eq: integrated inputs}
\end{equation}
where $w_{ij}^{l}$ is the synaptic weight and $b_{i}^{l}$ is the bias. Spike emission is determined by the Heaviside step function $\mathcal{H}(\cdot)$, which outputs $1$ if the argument is non-negative and $0$ otherwise. The spike train $s_{i}^{l}[t]$ can equivalently be expressed as a set of discrete firing events:
\begin{equation}
    s_{i}^{l}[t]=\sum_\tau\mathbbm{1}_{\tau\in\mathbb{F}_{i}^{l}}
    \label{eq: spike train}
\end{equation}
where $\mathbbm{1}$ is the indicator function and $\mathbb{F}_{i}^{l}$ denotes the set of firing times defined by the threshold-crossing condition:
\begin{equation}
    \tau:\hat{u}_{i}^{l}[\tau]\ge v_{th}^l
    \label{eq: spike time}
\end{equation}
For clarity, the definitions of common symbols are provided in Table~\ref{tab: symbol}.

\begin{figure}[tb]
    \centering
    \includegraphics[width=\linewidth]{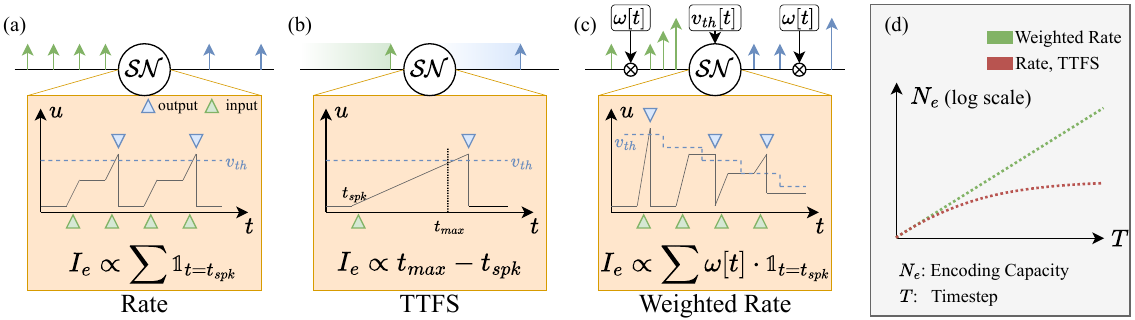}
    \caption{Spike representation of information. $\mathcal{SN}$ represents the spiking neuron model. The encoded information amount is given by $I_e$. The orange boxes illustrate how neurons decode information from spike patterns. (a) Rate coding encodes information in spike count or frequency. Each incoming spike contributes a fixed value (synaptic weight) to the membrane potential. (b) TTFS coding encodes information in the inverse of spike latency. A single spike triggers a sustained post-synaptic accumulation process. (c) More information can be encoded by applying temporal weights to incoming spikes. This can be viewed as a variation of rate coding where neuron adjusts its threshold according to the input weights. (d) Encoding capacity scales with timestep. The vertical axis $N_e$ (\textit{log scale}) indicates the number of encodable values.}
    \label{fig: spk repr}
\end{figure}

\subsection{Spike Representation}

Currently, rate coding is the most prevalent approach in SNNs for information representation due to its implementation simplicity \cite{Cao2015rate,Rueckauer2017rate,Hao2023cos,Bu2023qcfs,Hu2023fastsnn}. Under this framework, each spike triggers a discrete increment in membrane potential, and information is encoded in the spike count or frequency, as illustrated in Figure~\ref{fig: spk repr}(a).

To promote spike sparsity, alternative methods such as Time-To-First-Spike (TTFS) coding have been proposed \cite{Johansson2004ttfs,Yang2023lcttfs,Stanojevic2022exact}. TTFS encodes information in the inverse of spike latency: earlier spikes represent stronger stimuli. Neurons continuously integrate membrane potential upon receiving spikes, thereby extracting information from spike timing, as shown in Figure~\ref{fig: spk repr}(b).

\textit{However, both rate and TTFS coding schemes are inherently linear}, where spike count or latency scales proportionally with input intensity. This linear relationship constrains the growth of encoding capacity, as illustrated in Figure~\ref{fig: spk repr}(d), limiting the number of distinguishable values when few timesteps are available. To address this limitation, some works \cite{Rueckauer2021pattern,Christoph2021fs} have explored encoding information using weighted spikes. Typically, a set of exponentially decaying weights, such as $2^{T-t}$, is applied to inputs. Despite the improvements in encoding capacity, these direct weighting strategies increase system complexity and energy consumption. Challenges include the control and application of time-varying weights, as well as per-step adjustment of firing thresholds, as shown in Figure~\ref{fig: spk repr}(c).

\subsection{ANN-SNN conversion}

The core principle of ANN-SNN conversion involves the representation of ANN activations $a$ through discrete spike trains $s[t]$ \cite{Li2021cal,Rueckauer2017rate,Stanojevic2022exact}. \textit{This process can be interpreted as a layer-wise encoding-decoding procedure.} Specifically, given the input $s^{l-1}[t]$ encoding $a^{l-1}$ and SNN using the same weight $w^l$ as in the ANN, spiking neurons with appropriate dynamics decode the inputs and generate output spikes $s^l[t]$ that approximate $a^l$.

The spike representation strategy used to encode activations plays a critical role in determining the efficiency of the converted SNN. Improving the expressiveness of this representation is therefore central to enhancing conversion performance, and represents a key application scenario for our proposed method.


\begin{figure}[tb]
    \centering
    \includegraphics[width=\linewidth]{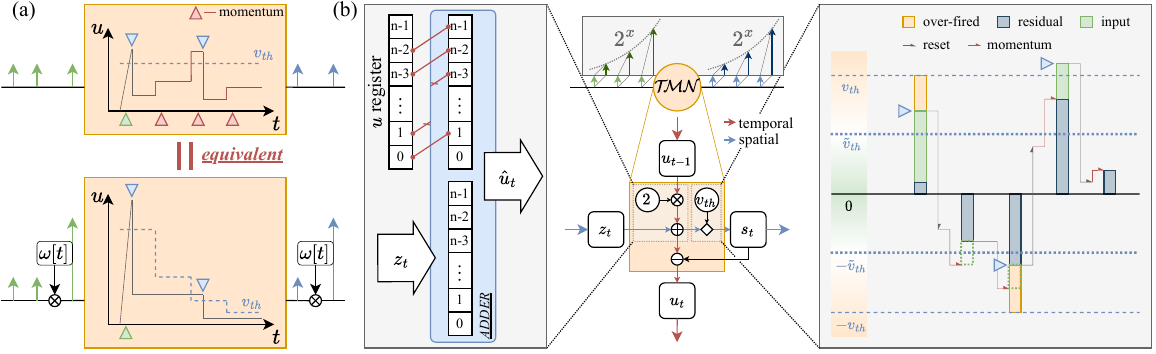}
    \caption{Ternary Momentum Neuron. (a) Momentum-based input weighting is mathematically equivalent to directly applying exponential weights. Our approach is preferable as it maintains a fixed threshold and relies on a simple stepwise weighting procedure. (b) Breakdown of the TMN model. Middle: TMN interprets the input as an exponentially weighted spike train while maintaining a simple computational graph. Left: Applying momentum is hardware-aware and can be implemented with negligible energy cost purely through wiring. The low bits of the membrane potential register are connected to the high bits of the adder input, performing a shift operation. Right: Predictive spiking nature of TMN. The predictive threshold pair $\pm\tilde{v}_{th}$ are used to determine spike firing, while $\pm{v}_{th}$ is used as reset amount. This design enables fast and accurate nonlinear encoding of dynamic inputs.}
    \label{fig: tmn}
\end{figure}

\section{Methods}

\subsection{Momentum-Based Weighting}

To enhance information encoding capacity, we adopt the following spike representation:
\begin{equation}
    a\propto\sum_t 2^{T-\tau}s[t]
    \label{eq: weighted repr}
\end{equation}
The weight $2^{T-t}$ is chosen because it enables a uniform representation of $a$ through combinations of spikes at different timesteps, akin to binary encoding. However, as illustrated in Figure~\ref{fig: spk repr}(c), directly applying weights incurs high implementation costs, due to the need for dynamic multipliers and threshold adjustment.

Recognizing that information extraction is ultimately determined by how neurons decode incoming sequences, we propose a momentum-based spiking neuron model for efficient input weighting. Its dynamics are defined as:
\begin{equation}
    \hat{u}[t]=\mathcal{D}(u[t-1],z[t])=2\times u[t]+z[t]
    \label{eq: momentum weight}
\end{equation}
The momentum term ($2\times$) amplifies previous residual membrane potential before adding current input, effectively modeling the \textit{relative importance} between inputs, which eliminates the need to adjust the firing threshold. We formally prove that this stepwise weighting is mathematically equivalent to directly receiving exponentially weighted spikes:

\begin{proposition}
    The momentum-based dynamics described in Equation~\ref{eq: momentum weight} are equivalent to integrating spikes weighted by an exponential kernel $2^{T-t}$.
    \label{prop: weight eqv}
\end{proposition}

The complete proof is provided in Appendix \ref{apsec: thm threshold setting proof}. Note this weighting mechanism is practically cost-free, as shown on the left side of Figure~\ref{fig: tmn}(b): The doubling operation can be implemented by directly wiring the low bits of the membrane potential register to the high bits of the adder input.

\begin{figure}[tb]
    \centering
    \includegraphics[width=\linewidth]{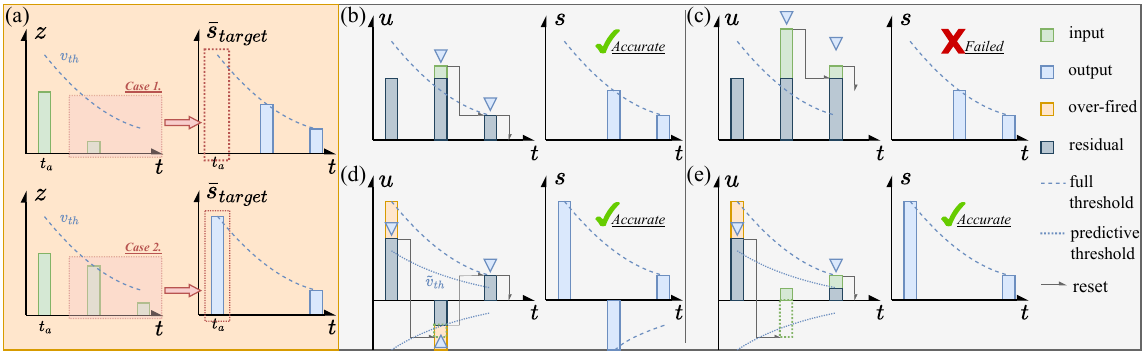}
    \caption{Challenges of neural computation with exponentially weighted spikes. \textit{Weights are directly applied for demonstration clarity.} (a) Temporal encoding uncertainty: Different spike decisions are required at $t_a$ depending on future (unknown) inputs. (b-c) Full-threshold spiking: (b) Successful encoding case. (c) Failure case due to underutilized high-weight spikes, resulting in residual information and encoding errors. (d-e) Proposed ternary predictive spiking. (d) Over-fired positive information is corrected by a negative spike. (e) Predictive spiking promotes better exploits the encoding capacity and achieves accurate result.}
    \label{fig: ps}
\end{figure}

\subsection{Ternary Predictive Spiking}

Exponentially decaying weights improve encoding capacity but make neural computation (encoding) more challenging due to (1) the uneven encoding capacity of spikes across timesteps and (2) uncertainty in future inputs. As illustrated in Figure~\ref{fig: ps}(a), \textit{unknown future inputs can influence the optimal spike encoding strategy at the current timestep.} A common failure case is shown in Figure~\ref{fig: ps}(c), where the high-weight spikes are underutilized, leaving substantial unencoded information that exceeds the capacity of lower-weighted spikes. The correct spike pattern is shown in the lower half of Figure~\ref{fig: ps}(a) for reference.

To address these limitations, we introduce a predictive spiking mechanism designed to minimize the residual membrane potential at each timestep. Specifically, \textit{the firing threshold is reduced to a predictive threshold $\tilde{v}_{th} = \frac{1}{2}v_{th}$, while the reset is still performed using the full threshold $v_{th}$.} A spike is emitted as soon as the membrane potential crosses $\tilde{v}_{th}$, anticipating future accumulation. To handle potential over-firing, a symmetric negative threshold $-\tilde{v}_{th}$ is introduced. These dual thresholds form a feedback mechanism that corrects over-firing and ensures accurate nonlinear encoding.

We prove that choosing $\tilde{v}_{th} = \frac{1}{2}v_{th}$ most effectively confines the residual membrane potential close to zero, ensuring minimal encoding error.

\begin{theorem}
    Let the integrated input $z_i^l[t]$ at each timestep be independent and follow a uniform distribution, $U(0,v_{th}^l)$. When $\tilde{v}_{th}=\frac{1}{2}v_{th}$, for all $t\in\{1,2,\ldots,T\}$, we have:
    \begin{equation}
        \notag
        \mathbb{E}(u_i^l[t])=0\quad\mathrm{and}\quad\mathbb{E}(u_i^l[t]^2)\mathrm{\ is\ minimized.}
    \end{equation}
    \label{thm: threshold setting}
\end{theorem}
The complete proof is provided in Appendix \ref{apsec: thm threshold setting proof}. The overall predictive spiking process is illustrated on the right side of Figure~\ref{fig: tmn}(b), and Figure~\ref{fig: ps}(d)(e) demonstrate two representative examples where the method effectively adapts to different input patterns.

\subsection{Pre-Charge Phase}

A widely adopted strategy to ensure encoding precision is to introduce a dedicated pre-charge phase, during which neurons integrate all (or most) of the inputs before generating output spikes. This concept and its variants are commonly employed across various encoding schemes, including TTFS coding \cite{park2020t2fsnn,Stanojevic2022exact,Yang2023lcttfs}, weighted rate coding \cite{Rueckauer2021pattern,Christoph2021fs}, and certain rate coding implementations \cite{Hao2023cos}. Pipelined computation architectures, such as Intel's Loihi \cite{Davies2018Loihi}, can be adopted to support this design.

Following this principle, we incorporate a pre-charge phase to further mitigate the temporal uncertainty of inputs. Unlike prior works that rely on a full $T$-step pre-charge phase, our predictive threshold shortens this phase, improving the pipeline efficiency. Empirically, we find that a one-step pre-charge phase is sufficient to achieve optimal performance in convolutional architectures, while two steps are preferred for Transformer-based structures.

With all components integrated, we arrive at the complete TMN model:
\begin{equation}
    \left\{
    \begin{aligned}
        & \hat{u}[t]=2\times u[t-1]+z[t] \hspace{2.5cm} & \text{(Momentum)} \\
        & s[t]=\mathcal{H}(\hat{u}[t]-\tilde{v}_{th})-\mathcal{H}(-\hat{u}[t]-\tilde{v}_{th}) & \text{(Ternary Predictive Spiking)}\\
        & u[t]=\hat{u}[t]-\mathbbm{1}_{t\notin\mathbb{P}}\cdot v_{th}s[t] & \text{(Pre-Charge Phase)}
    \end{aligned}
    \right.
    \label{eq: tmn}
\end{equation}
where $\mathbb{P}$ denotes the set of timesteps belonging to the pre-charge phase. The forward method of TMN is provided in Algorithm 1 in Appendix~\ref{apsec: tmn forward}. Given that the spike pattern generated by TMN shares structural and weighting similarities with the canonical signed digit representation, we name the associated spike representation Canonical Signed Spike (CSS) encoding.


\section{Experiments}

To evaluate the encoding capability of TMN, we convert pre-trained ANNs into CSS-coded SNNs. The detailed conversion process is provided in Appendix~\ref{apsec: conversion details}. We begin with image classification, the most common benchmark for SNNs, and conduct experiments on CIFAR-10, CIFAR-100, and ImageNet. Following this, we evaluate the energy consumption. To further assess the contribution of the TPS mechanism, we conduct ablation studies to verify its impact on improving pipeline throughput. We also empirically validate Theorem~\ref{thm: threshold setting} and visualize the encoding errors of TMN. Finally, to demonstrate the broader applicability of our method, we extend it to a diverse set of tasks, including CNN-based object detection, RoBERTa for natural language processing tasks, and Vision Transformer models.

\subsection{Classic Image Classification Benchmarks}
In Table~\ref{tab: img cls}, we compare the timestep requirements of various coding schemes for image classification benchmarks. The ``No Quant.'' column indicates whether a quantized ANN is required, while the ``No Cal.'' column marks whether post-conversion fine-tuning is applied. Works utilizing negative spikes are denoted with the ``$\dagger$'' symbol. Since conversion loss is a more informative evaluation metric than raw accuracy, we also report the accuracy of the source ANNs used in each study.

When directly encoding full-precision activations, CSS coding with TMN achieves near-lossless conversion with significantly fewer timesteps. For instance, on the ImageNet dataset, Li et al. \cite{Li2021cal} reported a conversion loss exceeding 1\% for ResNet-34 with 256 timesteps, whereas TMN achieved only 0.3\% conversion loss with just 8 timesteps. Although FS coding \cite{Christoph2021fs} also leverages weighted spikes, it incurs higher conversion loss despite requiring more timesteps.

Recent SOTA performance in SNNs \cite{Hu2023fastsnn, Hao2023cos} is typically achieved by encoding quantized activations with conventional rate coding. TMN integrates naturally with the quantization-based pipeline and achieves ANN-level accuracy with fewer timesteps than rate coding, enabling the development of higher-performance SNNs. For example, TMN achieved 76.23\% accuracy on CIFAR-100 using only 2 timesteps. 

While competitive rate coding implementation heavily relies on low-bit quantization, TMN enables high-precision, low-timestep SNNs without aggressive quantization or fine-tuning. As a notable case, on CIFAR-10, we converted a full-precision ResNet-18 using only 6 timesteps with a minimal conversion loss of 0.06\%, without introducing any additional training overhead.

\subsection{Energy Consumption Analysis}

Beyond its demonstrated temporal efficiency, TMN translates this benefit into substantial energy savings. We quantitatively evaluate the energy consumption\footnote{Based on \url{https://github.com/iCGY96/syops-counter}} of the CSS coding approach, with the results summarized in Table~\ref{tab: energy consumption}. Fast-SNN \cite{Hu2023fastsnn} and COS \cite{Hao2023cos} are chosen as strong baselines representing signed and plain rate coding, respectively. The results presented are obtained from our own implementations based on the official repositories, ensuring identical pre-conversion ANN accuracy. As shown in the table, our method consistently outperforms both baselines on large-scale benchmarks like ImageNet and simpler ones like CIFAR-10. Notably, TMN reduces energy consumption by around 50\% while achieving higher accuracy.

\begin{table}[tp]
\caption{Number of timesteps under different neural coding schemes, evaluated on the CIFAR-10, CIFAR-100 and ImageNet datasets. ``No Quant.'' column indicates whether quantization is required. ``No Cal.'' column marks whether post-conversion fine-tuning is required.}
\label{tab: img cls}
\centering
\resizebox{\linewidth}{!}{
    \begin{threeparttable}
        \begin{tabular}{c|llllllll}
        \toprule   
        & \bfseries Method & \bfseries\makecell[l]{No \\ Quant.} & \bfseries\makecell[l]{No \\ Cal.} & \bfseries Architecture & \bfseries\makecell[l]{ANN \\ Acc.} & \bfseries\makecell[l]{Coding \\ Scheme} & \bfseries Timestep & \bfseries\makecell[l]{SNN \\ Acc.} \\
        \midrule
        \bfseries\multirow{15}{*}{\rotatebox{90}{CIFAR-10}}
        & OPI \cite{Bu2022opi} & \greencheck & \greencheck & ResNet-20 & 92.74\% & rate & 64 & 92.57\% \\
        & FS-Conversion \cite{Christoph2021fs} & \greencheck & \greencheck & ResNet-20 & 91.58\% & FS & 10 & 91.45\% \\
        & SNN Calibration \cite{Li2021cal} & \greencheck & \redcross & VGG-16 & 95.72\% & rate & 128 & 95.65\% \\
        & TSC \cite{Han2020tsc} & \greencheck & \greencheck & VGG-16 & 93.63\% & TSC & 512 & 93.57\% \\
        & TTFS Mapping \cite{Stanojevic2023ttfs} & \greencheck & \redcross & VGG-16 & 93.68\% & TTFS & 64 & 93.69\% \\
        & LC-TTFS \cite{Yang2023lcttfs} & \greencheck & \greencheck & VGG-16 & 92.79\% & TTFS & 50 & 92.72\% \\
        \cmidrule{2-9}
        & \multirow{3}{*}{\bfseries TMN\tnote{$\dagger$}} 
        & \greencheck & \greencheck & ResNet-20 & 93.83\% & \bfseries\multirow{3}{*}{CSS} & \bfseries 7 & 93.73\% \\
        && \greencheck & \greencheck & VGG-16 & 95.90\% && \bfseries 8 & 95.92\% \\
        && \greencheck & \greencheck & ResNet-18 & 96.68\% && \bfseries 6 & 96.62\% \\
        \cmidrule{2-9}
        & QFFS \cite{Li2022quantization}\tnote{$\dagger$} & \redcross & \greencheck & ResNet-18 & 93.12\% & rate & 4 & 93.13\% \\
        & QCFS \cite{Bu2023qcfs} & \redcross & \greencheck & ResNet-18 & 96.04\% & rate & 16 & 95.92\% \\
        & COS \cite{Hao2023cos} & \redcross & \greencheck & ResNet-18 & 95.64\% & rate & 4 & 95.46\% \\
        & Fast-SNN \cite{Hu2023fastsnn}\tnote{$\dagger$} & \redcross & \redcross & ResNet-18 & 95.62\% & rate & 7 & 95.57\% \\
        \cmidrule{2-9}
        & \bfseries TMN\tnote{$\dagger$} & \redcross & \greencheck & ResNet-18 & 96.32\% & \bfseries CSS & \bfseries 3 & 96.34\% \\
        \midrule\midrule
        \bfseries\multirow{10}{*}{\rotatebox{90}{CIFAR-100\hspace{20pt}}}
        & OPI \cite{Bu2022opi} & \greencheck & \greencheck & ResNet-20 & 70.43\% & rate & 64 & 69.96\% \\
        & TSC \cite{Han2020tsc} & \greencheck & \greencheck & ResNet-20 & 68.72\% & TSC & 1024 & 67.81\% \\
        & SNN Calibration \cite{Li2021cal} & \greencheck & \redcross & VGG-16 & 77.89\% & rate & 128 & 77.40\% \\
        & TTFS Mapping \cite{Stanojevic2023ttfs} & \greencheck & \redcross & VGG-16 & 72.23\% & TTFS & 64 & 72.24\% \\
        & LC-TTFS \cite{Yang2023lcttfs} & \greencheck & \greencheck & VGG-16 & 70.28\% & TTFS & 50 & 70.15\% \\
        \cmidrule{2-9}
        & \multirow{2}{*}{\bfseries TMN\tnote{$\dagger$}} 
        & \greencheck & \greencheck & ResNet-20 & 71.21\% & \bfseries\multirow{2}{*}{CSS} & \bfseries 7 & 71.07\% \\
        && \greencheck & \greencheck & VGG-16 & 75.95\% && \bfseries 8 & 75.93\% \\
        \cmidrule{2-9}
        & QCFS \cite{Bu2023qcfs} & \redcross & \greencheck & VGG-16 & 76.28\% & rate & 16 & 76.24\% \\
        & COS \cite{Hao2023cos} & \redcross & \greencheck & VGG-16 & 76.28\% & rate & 4 & 76.26\% \\
        \cmidrule{2-9}
        & \bfseries TMN\tnote{$\dagger$} & \redcross & \greencheck & VGG-16 & 76.28\% & \bfseries CSS & \bfseries 2 & 76.23\% \\
        \midrule\midrule
        \bfseries\multirow{11}{*}{\rotatebox{90}{ImageNet\hspace{20pt}}}
        & TSC \cite{Han2020tsc} & \greencheck & \greencheck & ResNet-34 & 70.64\% & TSC & 4096 & 69.93\% \\
        & SNN Calibration \cite{Li2021cal} & \greencheck & \redcross & ResNet-34 & 75.66\% & rate & 256 & 74.61\% \\
        & TSC \cite{Han2020tsc} & \greencheck & \greencheck & VGG-16 & 73.49\% & TSC & 1024 & 73.33\% \\
        & OPI \cite{Bu2022opi} & \greencheck & \greencheck & VGG-16 & 74.85\% & rate & 256 & 74.62\% \\
        \cmidrule{2-9}
        & \multirow{2}{*}{\bfseries TMN\tnote{$\dagger$}}
        & \greencheck & \greencheck & ResNet-34 & 76.42\% & \bfseries\multirow{2}{*}{CSS} & \bfseries 8 & 76.10\% \\
        && \greencheck & \greencheck & VGG-16 & 75.34\% && \bfseries 8 & 75.17\% \\
        \cmidrule{2-9}
        & QFFS \cite{Li2022quantization}\tnote{$\dagger$} & \redcross & \greencheck & VGG-16 & 73.08\% & rate & 8 & 73.10\% \\
        & QCFS \cite{Bu2023qcfs} & \redcross & \greencheck & VGG-16 & 74.29\% & rate & 256 & 74.22\% \\
        & COS \cite{Hao2023cos} & \redcross & \greencheck & VGG-16 & 74.19\% & rate & 16 & 74.09\% \\
        & Fast-SNN \cite{Hu2023fastsnn}\tnote{$\dagger$} & \redcross & \redcross &VGG-16 & 73.02\% & rate & 7 & 72.95\% \\
        \cmidrule{2-9}
        & \bfseries TMN\tnote{$\dagger$} & \redcross & \greencheck & VGG-16 & 74.33\% & \bfseries CSS & \bfseries 5 & 74.32\% \\
        \bottomrule
        \end{tabular}
    \begin{tablenotes}
        \footnotesize
        \item[$\dagger$] Utilize negtive spikes
    \end{tablenotes}
    \end{threeparttable}
    }
\end{table}

\begin{table}[tbh]
\caption{Energy consumption of VGG-16 on CIFAR-10 and ImageNet. The results for COS and Fast-SNN are obtained using their open-source code, ensuring the same pre-conversion ANN accuracy.}
\label{tab: energy consumption}
\centering
\resizebox{0.9\textwidth}{!}{
    \begin{threeparttable}
        \begin{tabular}{c|lllllll}
        \toprule
        & \bfseries Method & \bfseries\makecell[l]{Coding \\ Scheme} & \bfseries Timestep & \bfseries Accuracy & \bfseries SyOPs (ACs) & \bfseries MACs & \bfseries\makecell[l]{Energy \\ Consumption}\\
        \midrule
        \bfseries\multirow{4}{*}{\rotatebox{90}{CIFAR-10\hspace{2pt}}}
        & ANN & - & - & 95.61\% & - & 313.60M & 1.4426mJ \\
        & COS \cite{Hao2023cos} & rate & 6 & 95.51\% & 258.35M & 11.01M & 0.2918mJ \\
        & Fast-SNN \cite{Hu2023fastsnn}\tnote{$\dagger$} & rate & 7 & 95.62\% & 295.1M & 12.85M & 0.3247mJ \\
        \cmidrule{2-8}
        & \bfseries TMN\tnote{$\dagger$} & \bfseries CSS & 3 & 95.65\% & 149.83M & 5.51M & \bfseries 0.1602mJ \\
        \midrule\midrule
        \bfseries\multirow{4}{*}{\rotatebox{90}{ImageNet\hspace{3pt}}}
        & ANN & - & - & 74.33\% & - & 15.49G & 71.25mJ \\
        & COS \cite{Hao2023cos} & rate & 7 & 73.69\% & 16.97G & 629.41M & 18.17mJ \\
        & Fast-SNN \cite{Hu2023fastsnn}\tnote{$\dagger$} & rate & 8 & 73.54\% & 19.28G & 719.32M & 20.77mJ \\
        \cmidrule{2-8}
        & \bfseries TMN\tnote{$\dagger$} & \bfseries CSS & 3 & 73.72\% & 8.69G & 269.75M & \bfseries 9.062mJ \\
        \bottomrule
        \end{tabular}
    \begin{tablenotes}
        \footnotesize
        \item[$\dagger$] Utilize negtive spikes
    \end{tablenotes}
    \end{threeparttable}
}
\end{table}

\begin{table}[tbh]
\caption{CSS coding can be generalized to diverse architectures and tasks. Entries such as $n$-bit and Level-$n$ in the `Architecture' column refer to the quantization precision of the original ANN. Results for Fast-SNN and SpikeZIP-TF are reproduced using their publicly available codebases.}
\label{tab: combo}
\centering
\resizebox{\linewidth}{!}{
    \begin{threeparttable}
        \begin{tabular}{l|llllll}
        \toprule
        \multicolumn{7}{c}{\bfseries Obeject Detection} \\ 
        \midrule 
        \bfseries Dataset &\bfseries Method & \bfseries Architecture & \bfseries ANN mAP & \bfseries\makecell[l]{Coding \\ Scheme} & \bfseries Timestep & \bfseries SNN mAP \\
        \midrule
        \bfseries\multirow{4}{*}{VOC2007\vspace{-6pt}}
        & \multirow{2}{*}{Fast-SNN \cite{Hu2023fastsnn}} & YOLOv2 (ResNet-34-3b) & 75.27 & \multirow{2}{*}{rate} & 7 & 73.43 \\
        && YOLOv2 (ResNet-34-4b) & 76.16 && 15 & 76.05 \\
        \cmidrule{2-7}
        & \multirow{2}{*}{\bfseries TMN} & YOLOv2 (ResNet-34-3b) & 75.27 & \multirow{2}{*}{\bfseries CSS} & \bfseries 3 & 75.20 \\
        && YOLOv2 (ResNet-34-4b) & 76.16 && \bfseries 4 & 76.18 \\
        \midrule
        \multicolumn{7}{c}{\bfseries Image Classification} \\ 
        \midrule 
        \bfseries Dataset & \bfseries Method & \bfseries Architecture & \bfseries ANN Acc. & \bfseries\makecell[l]{Coding \\ Scheme} & \bfseries Timestep & \bfseries SNN Acc. \\
        \midrule
        \bfseries\multirow{2}{*}{ImageNet\vspace{-6pt}}
        & SpikeZIP-TF \cite{you2024spikeziptf} & ViT-Small-Level32 & 81.56\% &rate & 64 & 81.45\% \\
        \cmidrule{2-7}
        & \bfseries TMN
        & ViT-Small-Level32 & 81.56\% & \bfseries CSS & \bfseries 5 & 81.51\% \\
        \midrule
        \multicolumn{7}{c}{\bfseries Sentiment Analysis} \\ 
        \midrule 
        \bfseries Dataset & \bfseries Method & \bfseries Architecture & \bfseries ANN Acc. & \bfseries\makecell[l]{Coding \\ Scheme} & \bfseries Timestep & \bfseries SNN Acc. \\
        \midrule
        \bfseries\multirow{2}{*}{SST-2\vspace{-6pt}}
        & SpikeZIP-TF \cite{you2024spikeziptf} & RoBERTa-Base-Level32 & 92.32\% & rate & 64 & 92.32\% \\
        \cmidrule{2-7}
        & \bfseries TMN & RoBERTa-Base-Level32 & 92.32\% & \bfseries CSS & \bfseries 5 & 92.32\% \\
        \midrule\midrule
        \bfseries\multirow{2}{*}{IMDB-MR\vspace{-6pt}}
        & SpikeZIP-TF \cite{you2024spikeziptf} & RoBERTa-Base-Level32 & 81.41\% & rate & 64 & 81.31\% \\
        \cmidrule{2-7}
        & \bfseries TMN & RoBERTa-Base-Level32 & 81.41\% & \bfseries CSS & \bfseries 5 & 81.36\% \\
        \bottomrule
        \end{tabular}
    \end{threeparttable}
}
\end{table}

\begin{figure}[tbh]
    \centering
    \includegraphics[width=\linewidth]{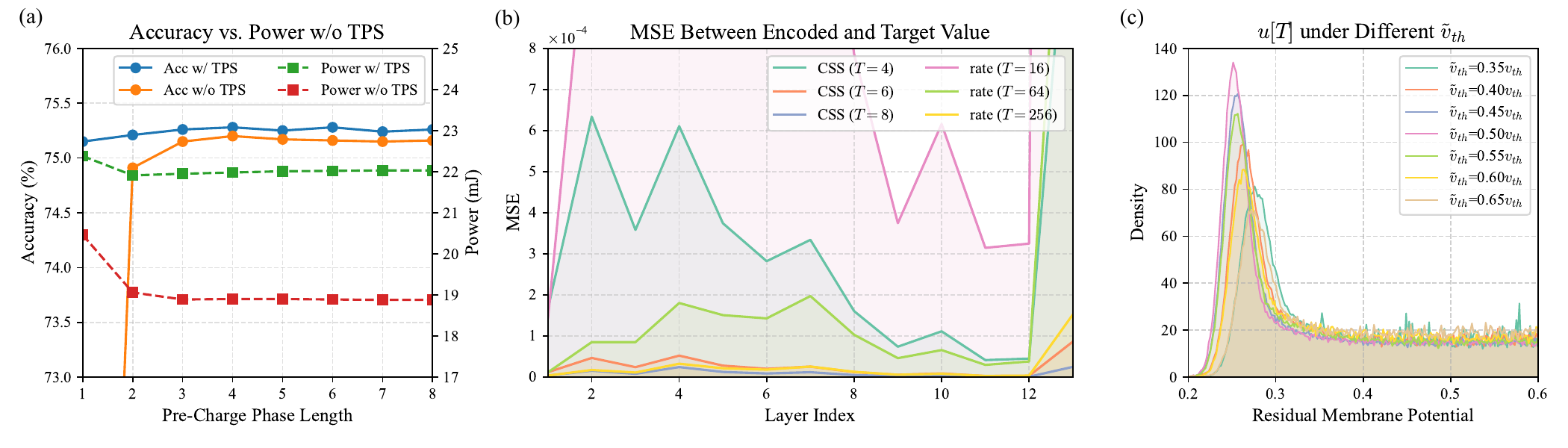}
    \caption{Ablation and Analysis of TMN. (a) Joint analysis of the TPS mechanism and the pre-charge phase. Solid lines correspond to the left y-axis (accuracy), and dashed lines correspond to the right y-axis (power). The TPS mechanism enables a reduction in pre-charge duration while improving model accuracy. However, predictive spiking introduces a small increase in energy consumption. (b)Layer-wise encoding error of CSS coding and rate coding under different timestep settings. Three timestep settings are selected for each scheme, with the timesteps for rate coding being exponential multiples of those for CSS. TMN exhibits an exponential advantage over linear coding schemes such as rate coding. (c) Impact of predictive threshold on the distribution of residual membrane potential. Setting $\tilde{v}_{th}=\frac{1}{2}v_{th}$ most effectively constrains the membrane potential within a narrow range, confirming the effectiveness of predictive spiking.}
    \label{fig: ablation}
\end{figure}

\subsection{Ablation and Theoretical Validation of TMN}

The TMN model consists of three important components: momentum-based weighting, the TPS mechanism, and the pre-charge phase. we conduct an ablation study to assess their individual and joint contributions. In Figure~\ref{fig: ablation}(a), we vary both the pre-charge length $P$ and the activation of the TPS mechanism to investigate their combined impact on TMN’s performance. From the accuracy curves, introducing TPS permits a shorter pre-charge ($P=1$, blue solid line) to match the performance of a longer pre-charge ($P=3$, orange solid line), which effectively improves system throughput under a pipelined setup. Moreover, predictive spiking better constrains the encoding error, as indicated by the blue solid line consistently staying above the orange solid line. However, predictive spiking slightly increases energy consumption, as shown by the dashed lines. This illustrates a trade-off between accuracy/throughput and energy efficiency.

In Figure~\ref{fig: ablation}(b), we demonstrate TMN's exponential enhancement in encoding capability by comparing it against a straightforward rate-coding scheme. For each coding approach, we vary the number of timesteps $T$ and plot the resulting mean squared error (MSE) of the encoding versus network depth. The results show that TMN achieves lower mean squared error with $T$ steps than rate coding obtains with $2^T$ steps, demonstrating the efficiency of the momentum-based weighting. Figure~\ref{fig: ablation}(c) offers empirical support for Theorem~\ref{thm: threshold setting}. We introduce perturbations in the predictive threshold $\tilde{v}_{th}$ around $\frac{1}{2}v_{th}$, and then plot the distribution of the average residual membrane potential $u[T]$ across all neurons under each threshold variant. The results show that setting $\tilde{v}_{th}=\frac{1}{2} v_{th}$ most effectively constrains the residual membrane potential, consistent with our theoretical analysis.

\subsection{TMN for Diverse Tasks and Architectures}

Additional experiments are conducted across various tasks to demonstrate the scalability and versatility of TMN. We benchmark our method against SOTA rate coding implementations using identical ANN weights, isolating the effect of the coding scheme. In object detection tasks, CSS coding not only reduces the number of timesteps required for activation encoding but also alleviates conversion error. For instance, TMN achieves an mAP improvement of 1.77 points (2.4\%) over Fast-SNN using a YOLOv2 framework with 3-bit ResNet-34 backbone,  Similar advantages are observed in ViT-based image classification and RoBERTa-based sentiment analysis tasks: TMN enables exponential compression of timesteps while maintaining or even exceeding baseline performance.

The fundamental nature of spike representation in SNNs suggests broad applicability of our method across various ANN-SNN conversion scenarios. A key advantage of TMN lies in its integration of weighting process in the neuron's dynamics. This makes the transition from widely used rate coding to CSS coding straightforward---requiring only the replacement of IF neurons with TMNs.


\section{Conclusion}

In this work, we propose the Temporal Momentum Neuron, a novel spiking neuron model that features momentum-based weighting and ternary predictive spiking. This neuron generates CSS-coded patterns which achieves efficient nonlinear information representation. Through theoretical analyses, diverse task benchmarks and a series of ablation studies, we demonstrate that TMN significantly reduces the number of encoding timesteps without sacrificing accuracy. The predictive spiking mechanism further improves computation efficiency and supports a energy-performance trade-off by shortening the pre-charge phase. Moreover, we show that TMN can be integrated into various network architectures, including CNNs and Transformers, highlighting its scalability and general applicability for spike encoding. Overall, TMN provides a principled and practical path toward temporally compact, high-fidelity spike representations, paving the way for more efficient and powerful spiking neural networks.


\newpage
\bibliographystyle{splncs04}
\bibliography{refs}


\appendix
\theoremstyle{plain}
\newtheorem{theoremap}{Theorem}[section]
\newcounter{appendixeq}
\newenvironment{appendixeq}{\refstepcounter{appendixeq}\equation}{\tag{A\theappendixeq}\endequation}

\section{Related Works}
\label{apsec: related works}

\subsection{Spike Coding Schemes}
Current mainstream coding schemes in converted SNNs include rate coding and TTFS coding. 

Rate coding represents activity by the number of spikes within a time window. Early methods aimed at reducing conversion loss, such as weight normalization \cite{Diehl2015norm}, threshold rescaling \cite{Sengupta2019vthscale}, and soft-reset neurons \cite{Han2020rmp}. More recent work focuses on reducing timesteps by optimizing neuron parameters: Meng et al. \cite{Meng2022ttrbr} introduced the threshold tuning method, while Bu et al. \cite{Bu2022opi} proposed optimizing the initial membrane potential. Additionally, recent works have explored quantizing the ANNs before conversion \cite{Bu2023qcfs,Hao2023cos,Hu2023fastsnn}. This approach directly reduces the number of activations that need to be mapped, providing an alternative way to minimize timesteps. 

Rueckauer \& Liu \cite{Rueckauer2018ttfs} were the first to attempt converting an ANN to a TTFS-based SNN, achieving increased sparsity but with significant conversion errors. Stanojevic et al. \cite{Stanojevic2022exact} showed that exact mapping is feasible. Yang et al. \cite{Yang2023lcttfs} improved this by using dynamic neuron threshold and weight regularization and completed the conversion with 50 timesteps per layer. Han \& Roy \cite{Han2020tsc} introduced the Temporal-Switch-Coding (TSC) scheme, where the time interval between two spikes encodes activation. However, in the above approaches, the improvement in encoding precision relies on a linear increase in the number of timesteps, which limits the performance of the converted SNN at low timesteps.

Using weighted spikes to represent activation values remains an underexplored area. St\"ockl \& Maass \cite{Christoph2021fs} and Rueckauer \& Liu \cite{Rueckauer2021pattern} employed spikes to encode the ``1''s in the binary represented activations. However, both approaches require neurons to wait for all input spikes before firing, resulting in high output latency. Kim et al. \cite{Jaehyun2018phase} sought to reduce encoding errors by repeatedly applying inputs, which requires thousands of timesteps. In contrast, by introducing TPS mechanism, we achieve fast and accurate computation.

\subsection{Negtive Spikes}
The role of negative spikes in SNNs has been widely studied in recent works. In rate-based ANN-SNN conversion \cite{Li2022quantization, Wang2022signed, Hu2023fastsnn}, negative spikes have been utilized to enhance neuron adaptability to input fluctuations, which improves conversion accuracy. Guo et al. \cite{Guo2023ternary} proposed ternary SNNs to increase the expressive capacity of spike sequences and trained them directly using gradient descent. Unlike existing approaches, we intentionally design neuron over-firing and employ negative spikes for correction. We demonstrate that this strategy facilitates efficient computation with weighted spikes.

\section{Mathematical Proofs}

\subsection{Proof of Proposition~\ref{prop: weight eqv}}
\begin{proposition}
The momentum-based dynamics described in Equation~\ref{eq: momentum weight} are equivalent to integrating spikes weighted by an exponential kernel $2^{T-t}$.
\end{proposition}
\begin{proof}
We proof for the $i$-th neuron in the $l$-th layer. The specific form of the operator $\mathcal{D}$ in Equation~\ref{eq: spiking neuron model} is defined by Equation~\ref{eq: momentum weight}. With the initial condition $u_i^l[0]=0$, we can write:
\begin{appendixeq}
    \hat u_i^l[1]=z_i^l[1]
\end{appendixeq}
Substituting $\hat u_i^l[1]$ into the third equation of Equation~\ref{eq: spiking neuron model}, we have:
\begin{appendixeq}
    u_i^l[1] = z_i^l[1] - v_{th}^l s_i^l[1]
\end{appendixeq}
Next, we derive the expression for $u_i^l[2]$ by substitute the above into Equation~\ref{eq: spiking neuron model}:
\begin{appendixeq}
    u_i^l[2]=2\times(z_i^l[1]-v_{th}^ls_i^l[1])+z_i^l[2]-v_{th}^ls_i^l[2]
\end{appendixeq}
We can generalize this process to iteratively compute the membrane potential up to $t=T$:
\begin{appendixeq}
    u_i^l[T]=\sum_{t=1}^T2^{T-t}(z_i^l[t]-v_{th}^ls[t])
\end{appendixeq}
Substituting $z_i^l[t]$ with Equation~\ref{eq: integrated inputs} and reorganizing the terms, we get:
\begin{appendixeq}
    u_i^l[T]=\sum_i w_{ij}^l\sum_{t=1}^{T}2^{T-t}v_{th}^{l-1}s_{j}^{l-1}[t]+\sum_{t=1}^{T}2^{T-t}b_{i}^l-\sum_{t=1}^{T}2^{T-t}v_{th}^ls_{i}^l[t]
\label{eq: stepwise weighting equivalence}
\end{appendixeq}
Based on Equation~\ref{eq: stepwise weighting equivalence}, we can conclude that the momentum-based weighting is equivalent to integrating spikes weighted by an exponential kernel $2^{T-t}$.
\end{proof}

\subsection{Proof of Theorem~\ref{thm: threshold setting}}
\label{apsec: thm threshold setting proof}

\begin{theorem}
Let the integrated input $z_i^l[t]$ at each timestep be independent and follow a uniform distribution, $U(0,v_{th}^l)$. When $\tilde{v}_{th}=\frac{1}{2}v_{th}$, for all $t\in\{1,2,\ldots,T\}$, we have:
\begin{equation}
\notag
\mathbb{E}(u_i^l[t])=0\quad\mathrm{and}\quad\mathbb{E}(u_i^l[t]^2)\mathrm{\ is\ minimized.}
\end{equation}
\notag
\end{theorem}

\begin{proof}
Let $\alpha=\nicefrac{\tilde{v}_{th}}{v_{th}}$. Let $p(z)$ denote the probability density function of the input $z_i^l[t]$. Define $k_i^l[t]=2\times u_i^l[t]$. For simplicity, we will drop both neuron and layer index and denote $z[t]$ and $k[t]$ as $z$ and $k$, respectively. According to Equation~\ref{eq: tmn}, we have:
\begin{appendixeq}
\begin{aligned}
\mathbb{E}(u[t+1])& =\mathbb{E}_k\left(\int_{-\infty}^{\infty}(z+k)p(z)\mathrm{d}z-v_{th}^l\int_{\alpha v_{th}^l-k}^{\infty}p(z)\mathrm{d}z\right) \\
& = \mathbb{E}_k\left(E(z)+k-v_{th}^l\int_{\alpha v_{th}^l-k}^{\infty}p(z)\mathrm{d}z\right) \\
& = \mathbb{E}(z)+\mathbb{E}(k)-v_{th}^l\mathbb{E}_k\left(\int_{\alpha v_{th}^l-k}^{\infty}p(z)\mathrm{d}z\right)
\end{aligned}
\end{appendixeq}
Let $q(k)$ denote the probability density function of $k$. We can write:
\begin{appendixeq}
\begin{aligned}
\mathbb{E}_k\left(\int_{\alpha v_{th}^l-k}^{\infty}p(z)\mathrm{d}z\right)& =\int_{-\infty}^{(\alpha-1)v_{th}^l}q(k)\int_{\alpha v_{th}^l-k}^{\infty}p(z)\mathrm{d}z\mathrm{d}k \\
& \hspace{40pt}+\int_{(\alpha-1)v_{th}^l}^{\alpha v_{th}^l}q(k)\int_{\alpha v_{th}^l-k}^{\infty}p(z)\mathrm{d}z\mathrm{d}k+\int_{\alpha v_{th}^l}^{\infty}q(k)\int_{\alpha v_{th}^l-k}^{\infty}p(z)\mathrm{d}z\mathrm{d}k \\
& =\int_{(\alpha-1)v_{th}^l}^{\alpha v_{th}^l}q(k)\int_{\alpha v_{th}^l-k}^{\infty}p(z)\mathrm{d}z\mathrm{d}k+\int_{\alpha v_{th}^l}^{\infty}q(k)\mathrm{d}k \\
& = \mathbb{\tilde E}_k\left(F_z(\infty)-F_z(\alpha v_{th}^l-k)\right)+F_k(\infty)-F_k(\alpha v_{th}^l)
\end{aligned}
\end{appendixeq}
where $F_z(\cdot)$ and $F_k(\cdot)$ denote the cumulative distribution functions of $z$ and $k$, respectively. $\mathbb{\tilde E}_k$ denotes the expectation of $k$ within the target range. Note that $Z\sim U(0,v_{th}^l)$, so $F_z(\cdot)$ is linear. We further assume that $k$ is almost constrained within the threshold, i.e., $F_k(\alpha v_{th}^l)\approx 1$. Therefore, we have:
\begin{appendixeq}
\mathbb{\tilde E}_k\left(\int_{\alpha v_{th}^l-k}^{\infty}p(z)\mathrm{d}z\right)\approx 1-F_z(\alpha v_{th}^l-\mathbb{\tilde E}(k))\approx 1-F_z(\alpha v_{th}^l-\mathbb{E}(k))
\end{appendixeq}
and
\begin{appendixeq}
\mathbb{E}(u[t+1])=\mathbb{E}(z)+\mathbb{E}(k)-v_{th}^l\left(1-F_z(\alpha v_{th}^l-\mathbb{E}(k))\right)
\label{eq: exp iter}
\end{appendixeq}
Note that $k[0]=0$ and $F_z\left(\frac{1}{2}v_{th}^l\right)=\frac{1}{2}$. When $\alpha=\frac{1}{2}$, we have:
\begin{equation}
\mathbb{E}(u[1])=\mathbb{E}(z)-v_{th}^l\left(1-F_z(\frac{1}{2}v_{th}^l)\right)=0=\mathbb{E}(k[1])
\end{equation}
By repeatedly applying Equation~\ref{eq: exp iter}, we can conclude that for all $t\in\{1,2,\ldots,T\}$, $\mathbb{E}(u_i^l[t])=0$.

Similarly, we can write:
\begin{appendixeq}
\begin{aligned}
\mathbb{E}(u[t+1]^2)& =\mathbb{E}_k\left(\int_{-\infty}^{\alpha v_{th}^l-k}(z+k)^2p(z)\mathrm{d}z+\int_{\alpha v_{th}^l-k}^{\infty}(z+k-v_{th}^l)^2p(z)\mathrm{d}z\right) \\
& =\mathbb{E}_k\left(\int_{-\infty}^{\alpha v_{th}^l}z^2p(z-k)\mathrm{d}z+\int_{\alpha v_{th}^l}^{\infty}(z-v_{th}^l)^2p(z-k)\mathrm{d}z\right) \\
& =\mathbb{E}_k\left(\int_{-\infty}^{\infty}z^2p(z-k)\mathrm{d}z-v_{th}^l\int_{\alpha v_{th}^l}^{\infty}(2z-v_{th}^l)p(z-k)\mathrm{d}z\right)
\end{aligned}
\end{appendixeq}
Taking the derivative of the above equation with respect to $\alpha$ and exchanging the order of differentiation and integration (i.e., $\mathbb{E}(\cdot)$), we obtain:
\begin{appendixeq}
\begin{aligned}
\frac{\partial\mathbb{E}(u[t+1]^2)}{\partial\alpha}& =-\mathbb{E}_k\left(\frac{\partial}{\partial\alpha}v_{th}^l\int_{\alpha v_{th}^l}^{\infty}(2z-v_{th}^l)p(z-k)\mathrm{d}z\right) \\
& = \mathbb{E}_k\left(\frac{\partial}{\partial\alpha}v_{th}^l\int^{\alpha v_{th}^l}_{\infty}(2z-v_{th}^l)p(z-k)\mathrm{d}z\right) \\
& =(2\alpha-1)(v_{th}^l)^3\mathbb{E}_k\left(p(\alpha v_{th}^l-k)\right)
\end{aligned}
\end{appendixeq}
Note that $\mathbb{E}_k\left(p(\alpha v_{th}^l-k)\right)>0$. The derivative is negative when $\alpha<\frac{1}{2}$ and positive when $\alpha>\frac{1}{2}$. Therefore, when $\alpha=\frac{1}{2}$, $\mathbb{E}(u_i^l[t]^2)$ is minimized for all $t\in\{1,2,\ldots,T\}$.
\end{proof}

\section{Forward method of the TMN neuron}
\label{apsec: tmn forward}

\begin{algorithm}[h]
\caption{Forward method of the TMN neuron}
\label{alg: tmn forward}
\begin{algorithmic}
\State{\bfseries Input:} input $X$ of shape [B, C, H, W], current timestep $t$, pre-charge length $P$, reset signal $R$, full threshold $V$ 
\State{\bfseries Output:} output spike train $S$ of shape [B, C, H, W]
\If{$R==1$}
    \State Membrane potential $U\gets\mathtt{zeros\_like}\left(X\right)$
    \State Predictive threshold $\tilde V\gets\frac{1}{2}\cdot2^P V$
\EndIf
\If{$t\le P$}
    \State $U\gets 2\times U + X$\hspace{15pt}\texttt{/*P-step pre-charge*/}
\Else
    \State $U\gets 2\times U + X$
    \State $S\gets(U\ge \tilde{V})\mathtt{.float}()-(U\le -\tilde{V})\mathtt{.float}()$\hspace{15pt}\texttt{/*fire ternary spikes*/}
    \State $U\gets U-2^PVS[i]$\hspace{15pt}\texttt{/*soft reset*/}
\EndIf
\State $S\gets VS$
\end{algorithmic}
\end{algorithm}

\section{ANN-SNN conversion}
\label{apsec: conversion details}

The proposed method eliminates the need for directly applying weights to input pixel values, as weighting is integrated during neural computation. Therefore, we adopt the widely used direct coding for input static images \cite{Rueckauer2017rate, Li2021cal, Hao2023cos, Hu2023fastsnn}: the analog input activations of the first hidden layer are interpreted as constant currents, with spiking outputs starting from this layer.

We replace the ReLU activation function with the TMN neuron to encode the hidden layer activations. Note that since TMN neurons can encode negative activations, additional logic is required to zero out sequences encoding negative values.This logic simplifies to detecting sequences where the first spike is positive, which can be easily implemented.

To determine the full threshold $v_{th}^l$ for each layer, we use the strategy proposed by Rueckauer et al. \cite{Rueckauer2017rate}: after observing ANN activations over a portion of the training set, we calculate the 99.99th percentile $p^{l}$ of the activation distribution, and then set $v_{th}^{l}$ to $p^{l}$\footnote{More precisely, $v_{th}^l=p^l\cdot \frac{T}{\sum_{t=1}^T2^{T-t}}$. Since scaling the full threshold of each layer by the same value has no practical effect, we directly set $v_{th}^l$ to $p^l$ for simplicity.}. This approach improves the network's robustness to outlier activations. The pseudocode
for the conversion process is provided in Algorithm~\ref{alg: conversion}.

\begin{algorithm}[h]
\caption{Algorithm for ANN-SNN conversion under CSS coding.}
\label{alg: conversion}
\begin{algorithmic}
\State{\bfseries Input:} ANN model $f_A(\hat{W},\hat{b})$, number of timesteps $T$, number of batches $B$, layer count $L$.
\State{\bfseries Output:} SNN models $f_S(W,b)$
\State \texttt{/*determine $v_th^l*/$}
\For{$l=0$ {\bfseries to} $L-1$}
\State $\bar p^l\gets0$
\For{$n=0$ {\bfseries to} $B-1$}
\State $p^l\gets99.99$-th percentile of $a^l$ distribution
\State \texttt{/*take average*/}
\State $\bar p^l\gets\bar p^l+\nicefrac{p^l}{B}$
\EndFor
\State $\theta^l\gets\bar p^l$
\EndFor
\For{$l=0$ {\bfseries to} $L-1$}
\State \texttt{/*copy weight and bias*/}
\State $W^l\gets\hat W^l$
\State $b^l\gets\hat b^l$
\State Replace ReLU activation with TMN.
\EndFor
\end{algorithmic}
\end{algorithm}

\end{document}